\documentclass[twoside]{article}

\usepackage[accepted]{aistats2025}
\usepackage{appendix}
\usepackage{color}
\usepackage{url}
\usepackage{graphicx}
\usepackage{multirow}
\usepackage{amssymb}
\usepackage{amsthm}
\usepackage[round]{natbib}

\usepackage{hyperref}

\newcommand{\beq}{\vspace{0mm}\begin{equation}}
\newcommand{\eeq}{\vspace{0mm}\end{equation}}
\newcommand{\beqs}{\vspace{0mm}\begin{eqnarray}}
\newcommand{\eeqs}{\vspace{0mm}\end{eqnarray}}
\newcommand{\barr}{\begin{array}}
\newcommand{\earr}{\end{array}}

\newcommand{\av}[0]{{\boldsymbol{a}}}

\newcommand{\cv}[0]{{\boldsymbol{c}}}

\newcommand{\ev}[0]{{\boldsymbol{e}}}

\newcommand{\xv}{\boldsymbol{x}}

\newcommand{\zv}{\boldsymbol{z}}

\newtheorem{thm}{Theorem} [section]

\newtheorem{lem}[thm]{Lemma}
\newtheorem{prop}[thm]{Proposition}
\newtheorem{assum}{Assumption} [section]

\usepackage{booktabs}
\usepackage{array}
\newcolumntype{L}[1]{>{\raggedright\let\newline\\\arraybackslash\hspace{0pt}}m{#1}}
\newcolumntype{C}[1]{>{\centering\let\newline\\\arraybackslash\hspace{0pt}}m{#1}}
\newcolumntype{R}[1]{>{\raggedleft\let\newline\\\arraybackslash\hspace{0pt}}m{#1}}

\begin{document}

\twocolumn[

\aistatstitle{Elastic Representation: Mitigating Spurious Correlations for Group Robustness}

\aistatsauthor{ Tao Wen \And Zihan Wang \And Quan Zhang \And Qi Lei}

\aistatsaddress{Dartmouth College \\ \href{mailto:tw2672@nyu.edu}{tw2672@nyu.edu} \And  New York University \\ \href{zw3508@nyu.edu}{zw3508@nyu.edu} \And Michigan State University \\ \href{quan.zhang@broad.msu.edu}{quan.zhang@broad.msu.edu} \And New York University \\ \href{ql518@nyu.edu}{ql518@nyu.edu} } ]

\begin{abstract}
Deep learning models can suffer from severe performance degradation when relying on spurious correlations between input features and labels, making the models perform well on training data but have poor prediction accuracy for minority groups. This problem arises especially when training data are limited or imbalanced. While most prior work focuses on learning invariant features (with consistent correlations to y), it overlooks the potential harm of spurious correlations between features. We hereby propose Elastic Representation (ElRep) to learn features by imposing Nuclear- and Frobenius-norm penalties on the representation from the last layer of a neural network. Similar to the elastic net, ElRep enjoys the benefits of learning important features without losing feature diversity. The proposed method is simple yet effective. It can be integrated into many deep learning approaches to mitigate spurious correlations and improve group robustness. Moreover, we theoretically show that ElRep has minimum negative impacts on in-distribution predictions. This is a remarkable advantage over approaches that prioritize minority groups at the cost of overall performance.
\end{abstract}

\section{INTRODUCTION}\label{sec:intro}
Group robustness is critical for deep learning models, particularly when they are deployed in real-world applications like medical imaging and disease diagnosis \citep{huang2022developing,kirichenko2022last}. 
In practice, data are often limited, and models are frequently exposed to domains or distributions that are not well represented in training data. Group robustness aims to enable models to generalize to unseen domains, addressing challenges such as differing image backgrounds or styles. Standard training procedures, like empirical risk minimization, can result in good performance on average but poor accuracy for certain groups, especially in the presence of spurious correlations \citep{sagawa2020investigation,haghtalab2022demand}.

Spurious correlations arise when models rely on features that correlate with class labels in the training data but are irrelevant to the true labeling function. This leads to performance degradation for out-of-distribution (OOD) generalization. For example, a model trained to classify objects, like waterbirds and landbirds, might rely on background or textures \citep{geirhos2018imagenet,xiao2020noise}, like water and land, which are irrelevant to the object, resulting in low accuracy for minority groups of waterbirds on land and landbirds on water. 
Ideally, a deep learning model should learn features that have invariant correlations with labels for all distributions.

While neural-network classification models trained by empirical risk minimization (ERM) may lead to poor group robustness and OOD generalization \citep{geirhos2020shortcut,zhang2022correct} and be no better than random guessing on minority groups when predictions exclusively depend on spurious features \citep{shah2020pitfalls},  recent studies have shown that even standard ERM can well learn both spurious and invariant (non-spurious) features \citep{kirichenko2022last,izmailov2022feature}; the low accuracy of ERM on minority groups results from the classifier, i.e., the linear output layer of a neural network, which tends to overweight spurious features.
Based on this finding, we propose Elastic Representation (ElRep) 
by imposing nuclear-norm and Frobenius-norm penalties on feature representations. This approach not only regularizes the learning of spurious features but also enhances the prominence of invariant features.

Our approach borrows the idea from the elastic net \citep{zou2005regularization} that imposes $\ell_1$ and $\ell_2$ penalties on regression coefficients. 
Though we regularize the feature representation rather than the weights of the classifier, the intuition is similar. Specifically, a nuclear norm regularizing the singular values of the representation matrix facilitates a sparse retrieval of the backbone features, and its effectiveness has been underpinned by \citet{shi2024domain}. 
However, we have observed that using a nuclear-norm penalty alone can suffer from a problem similar to that of lasso, as it tends to capture only part of the invariant features but omit others if they are highly correlated.
The over-regularization can undermine the robustness on minority groups or unseen data where only the omitted features are present. 

\begin{figure}[!t]
\includegraphics[width=\linewidth]{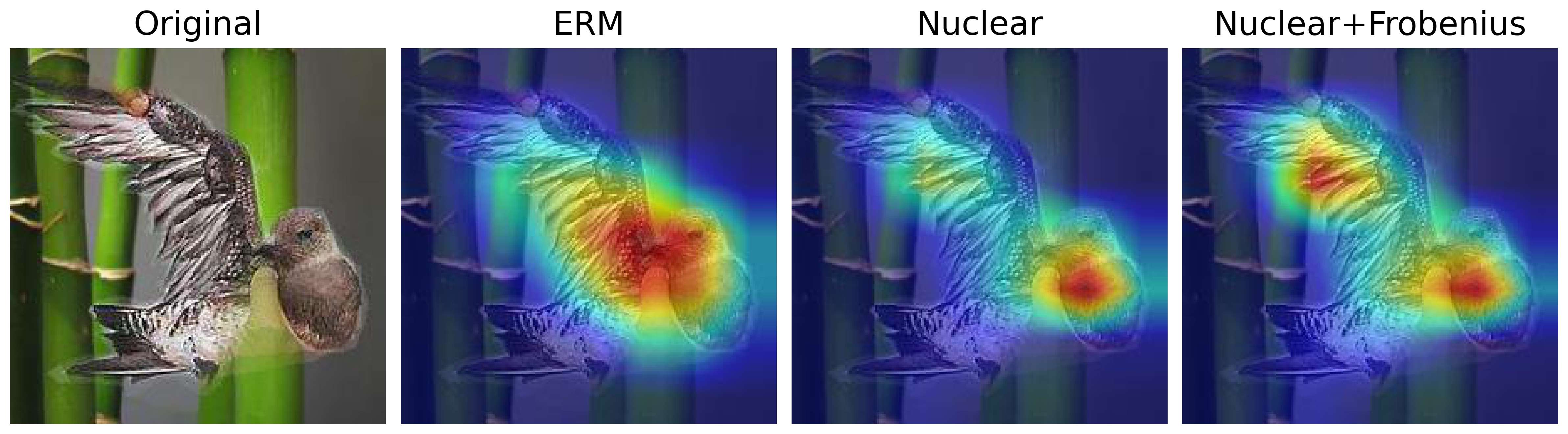}
\caption{A long-tailed Jaeger, a waterbird on a land background, from the waterbirds dataset \citep{sagawa2019distributionally}. The heat maps depict the pixel contributions to bird type prediction using Grad-CAM~\citep{Selvaraju_2019}. From left to right are the original image, ERM, ERM with nuclear norm, and ERM with nuclear and Frobenius norms, respectively. 
ERM learns features including background areas. ERM with nuclear norm focuses on the head, and ERM with both norms evenly emphasizes the head and the wing.}\label{fig:norms}
\end{figure}

To address this issue, we introduce a Frobenius-norm penalty to regularize the representation in addition to a nuclear-norm penalty. Analogous to the advantage of the elastic net over lasso, the Frobenius norm tunes down the sparsity and keeps more invariant features when they are correlated. 
We illustrate this finding in Figure~\ref{fig:norms} with an image of a waterbird on a land background. 
ERM without regularization captures features that include the object and background areas. When applying a nuclear norm, the learned features emphasize the bird's head but somewhat overlook the wing. 
So, the model may fail on images where a bird's head is blocked.
With both nuclear and Frobenius norms, the representation captures the head and wing, effectively regularizing the learning of the background and making both invariant features prominent without sacrificing either.

We distinguish ElRep from extant literature by making the following contributions. 
\begin{enumerate}
\item ElRep mitigates spurious correlation \emph{without relying on group information}, which is often required by many group robustness methods to adjust weights of minority groups. This is essential for real-world applications as group annotations are largely impractical.
\item We theoretically prove and empirically show that ElRep has a minimum sacrifice of the overall performance while \emph{improving the worst-group accuracy}. 
\item ElRep is simple yet effective without extra computational cost. It is a \emph{general framework} that can be combined with and further improve many state-of-the-art approaches. 
\end{enumerate}

The paper proceeds as follows. In Section~\ref{sec:literature}, We compare ElRep and related work for group robustness. 
In Section~\ref{sec:method}, we formally introduce the proposed method. 
In Section~\ref{sec:exp}, we use synthetic and real data to showcase the outstanding performance and favorable properties of ElRep. Section~\ref{sec:thy}  theoretically underpins ElRep, and
Section~\ref{sec:conclusion} concludes the paper.

\section{RELATED LITERATURE}\label{sec:literature}
Extensive efforts have been made to mitigate spurious correlations. Two of the common practices are optimization-based methods addressing group imbalance and via improved learning of invariant features. 
Our ElRep framework can be combined to improve an optimization-based method. It also supplements the representation learning literature with a much simpler procedure based on the finding that ERM already learns invariant features.
We review the literature in these two streams and refer readers to \citep{ye2024spurious} for a comprehensive taxonomy of extant popular approaches.
 
\looseness-1 Neural networks relying on spurious correlations often suffer from degradation of performance consistency across different groups or subpopulations. A typical reason is selection biases on
datasets \citep{ye2024spurious}, where groups are not equally represented. Imbalanced groups can lead neural networks to prioritize the majority and learn their spurious correlations that may not hold for the minority. A considerable amount of work addresses group imbalance by utilizing the group information for distributionally
robust optimization (DRO) to improve performance in worst cases. 
For example, groupDRO \citep{sagawa2019distributionally} minimizes the worst-group loss instead of the average loss, and there emerges subsequent work also emphasizing minority groups in training \citep[e.g.,][]{goel2020model,levy2020large,sagawa2020investigation,zhang2020coping,deng2024robust}. 
Notably, \citet{idrissi2022simple} show that simple group balancing by subsampling or reweighting achieves state-of-the-art accuracy, highlighting the importance of group information.

Though these approaches have improved worst-case accuracy, they rely on group annotations that are often impractical in real-world applications. Methods that automatically identify minority groups are developed. For example, one can use a biased model to find hard-to-classify data, treat them as a minority group, and then use a downstream model to improve the accuracy on the ``minority'' for group robustness \citep{nam2020learning,liu2021just,yenamandra2023facts}. 
These approaches train the models twice and may be computationally expensive.
To improve the efficiency, \citet{du2022less}, \citet{moayeri2023spuriosity}, and \citet{yang2024identifying} find data points or groups with high degrees of spuriosity in an early stage of training and then mitigate the model's reliance on them. 
Overall, the group information, either manually annotated or automatically identified, plays a crucial role in this stream of research that tries to address group imbalance. In stark comparison, ElRep does not require group information and is readily integrated into many of these optimization-based methods 
to further improve the performance.

Research in representation learning tries to better
understand the underlying relationships between variables, capture improved features, 
and make models more resilient to spurious correlations \citep[e.g.,][]{sun2021recovering,veitch2021counterfactual,eastwood2024spuriosity}. 
Recent studies \citep{kirichenko2022last,izmailov2022feature,rosenfeld2022domain,chen2024understanding,zhong2024bridging} potentially make representation learning easier by showing that ordinary ERM can learn both spurious and invariant feature representation. This implies that one can efficiently improve group robustness by downplaying spurious features and highlighting invariant features, without the need to explore causal relationships, making the process conceptually and computationally much simpler.

Based on this finding, \citet{kirichenko2022last} and \citet{izmailov2022feature} retrain the last layer of a neural network on a small held-out dataset where the
spurious correlation breaks. However, this method requires the group information. To avoid group annotations, one can combine the idea of automatic identification of ``minority groups'' and the last-layer fine-tuning. For example, \citet{chen2024understanding} alternately use easy- and hard-to-classify data to enforce the learning of richer features in the last layer. 
Similarly,  \citet{labonte2024towards} propose using disagreements between the ERM and early-stopped
models to balance the classes in the last-layer fine-tuning.

Since ERM can well learn both spurious and invariant features, a natural way for group robustness is to mitigate spurious correlations through regularization. However, this approach has not been thoroughly explored. 
We fill this research gap by imposing nuclear- and Frobenius-norm penalties to achieve a good balance between pruning spurious features and keeping invariant features.
A closely related study \citep{shi2024domain} uses a nuclear-norm regularization for parsimonious representation. However, as illustrated in Figure~\ref{fig:norms}, it may suffer from over-regularization and losing invariant features. ElRep introduces a Frobenius norm to alleviate this problem. Theoretically, this will maintain the in-distribution (ID) performance while making the invariant feature less sparse. Empirically, it outperforms using a nuclear norm alone and further improves state-of-the-art approaches when combined with them. 

\section{METHODOLOGY}\label{sec:method}
\subsection{Preliminaries and Notations}
We consider the setting where the domains of training and testing are different. We have $(\xv,y)\sim\mathcal{D}_{\mathrm{id}}$ for training data and $(\xv,y)\sim\mathcal{D}_{\text{ood}}$ for test data. The model we consider is $f(\xv)=W^\top\Phi(\xv)$, where $\Phi$ is a latent representation function. 
Our goal is to train the model with data from $\mathcal{D}_{\text{id}}$ and reduce the risk $\mathbb{E}_{(\xv,y)\sim\mathcal{D}_{\text{ood}}}\left[\ell(f(\xv),y)\right]$ on the test domain, where $\ell$ is a loss function. 
To achieve this goal, the representation $\Phi$ is trained to extract features of the input data. The features that generate data $\xv$ include invariant and spurious features, with the former only related to the label $y$ and the latter also related to the environment. Since the environment domains are different between the training and testing distributions, a good $\Phi$ should preserve invariant features and remove spurious features.
We use $\mathcal{L}(W,\Phi)$ to represent some risk function on the training domain with respect to a weight matrix $W$ and representation $\Phi$, where we omit the loss function $\ell$. We use $\|\cdot\|_*$ to denote the nuclear norm of a matrix and $\|\cdot\|_{\text{F}}$  the Frobenius norm.
Specifically, $\|A\|_*=\mathrm{Tr}\left((A^\top A)^{1/2}\right)$ and $\|A\|_\mathrm{F}=\left(\mathrm{Tr}(A^\top A)\right)^{1/2}$.  For vectors, $\|\cdot\|_2$ denotes its $\ell_2$ norm.

\subsection{Elastic Representation}
In classification and regression tasks, models learn features from labeled data. In order to make better predictions for OOD data, the model should learn the features that highly correlate to the label. Invariant features should have a higher correlation than spurious features since the former preserves in both ID and OOD data but the latter only appears in ID data. A latent representation $\Phi$ contains both kinds of features. Our goal is to highlight the invariant and eliminate the spurious.


We consider the model $f(\xv)=W^\top\Phi(\xv)$ with a latent representation $\Phi(\xv)$. By minimizing $\mathcal{L}(W,\Phi)$, we can obtain label-related features. However, the spurious features are also preserved, which cannot help OOD prediction. According to \cite{shi2024domain}, by adding the nuclear norm of the representation $\Phi$, the information contained in $\Phi(\xv)$ is reduced. 
This regularization eliminates spurious features but meanwhile, could also rule out part of invariant features. 
By Elastic Representation (ElRep) that includes an extra Frobenius-norm regularization, we expect to capture more invariant features. The objective function is 
\begin{equation}
    \min_{W,\Phi}\mathcal{L}(W,\Phi)+\lambda_1\left\|\Phi(\xv)\right\|_*+\lambda_2\left\|\Phi(\xv)\right\|_{\mathrm{F}}^2,
\end{equation}
where $\lambda_1$ and $\lambda_2$ are hyper-parameters that control the intensity of the respective penalty. 
Note that this regularization can be added to a wide range of risk functions, for example ERM and GroupDRO \citep{sagawa2019distributionally}. For ERM, the risk function $\mathcal{L}(W,\Phi):=\mathbb{E}_{(\xv,y)\in \mathcal{D}_{\mathrm{in}}}\left[\ell(f(\Phi(\xv)),y)\right]$.

\subsection{Thought Experiment}
\begin{figure}[h]
\vspace{-2mm}
    \centering
    \includegraphics[width=0.97\columnwidth]{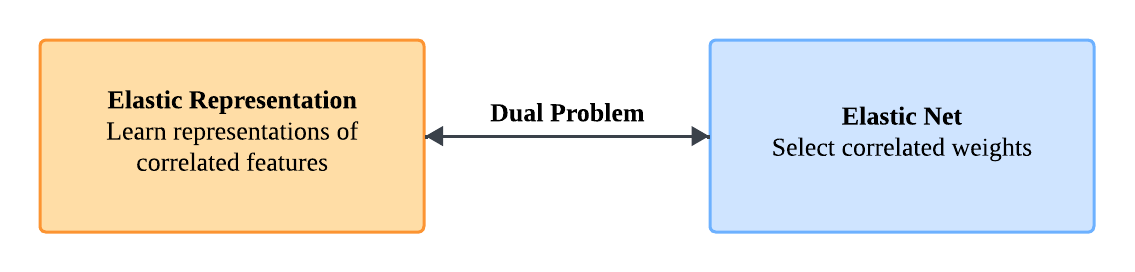}
        \vspace{-2mm}
    \caption{Connections between ElRep and elastic net. }
    \label{fig:diagram}
\end{figure}
To demonstrate the intuition behind the benefit of ElRep, we present a simple statistical thought experiment. First, regularizing on the representation $\Phi(x)$ is a dual problem to regularizing the weight $W$ (See Figure \ref{fig:diagram}):  Lasso or ElasticNet selects features by learning sparse model weight and thus zero-ing out the effect of the features corresponding to the zero weights. Meanwhile, nuclear norm or ElRep directly enforces learning low rank $\Phi(X)$ ((fewer number of features).  We illustrate the benefit of Elastic Net first. Consider two features $\Phi(x)_1$ and $\Phi(x)_2$ with a strong spurious correlation $\rho$ close to 1, but both features are equally important to predict $y$. If $\Phi(x)_1$ has a smaller magnitude, $\ell_1$ regularization will assign its associated weight $w_1$ to 0, while elastic net ($\ell_1+\ell_2$) tend to allocate non-zero elements in both $w_1$ and $w_2$ (since $\|[0.5,0.5]\|_2 < \|[0,1]\|_2$.) 
If for a target distribution the correlation between features changes, then $\ell_1$ regularization fails to utilize the information from $\Phi(x)_1$ to predict $y$. We defer a more precise analysis to Section~\ref{subsec:toy_theory}. Similarly, ElRep will also learn diverse features even if they might have some strong spurious correlation. 
Despite the connection between elastic net and ElRep, the latter is much better, since the success of elastic net depends on the quality of a pre-existing set of features to select from, and features learned through ERM 
may still have \emph{non-linear spurious correlations} or lack diversity. ElRep addresses these issues by directly learning more robust features.

\section{EXPERIMENTS}\label{sec:exp}
In this section, we evaluate the effectiveness of ElRep on both synthetic and real data. For synthetic data, we design a setting where our method demonstrates advantages in terms of loss minimization and sparsity. For real data, we consider three popular benchmark datasets in the presence of spurious features: CelebA \citep{liu2015deeplearningfaceattributes}, Waterbirds \citep{sagawa2019distributionally}, and CivilComments-WILDS \citep{koh2021wildsbenchmarkinthewilddistribution}. We present the worst-group accuracy, which assesses the minimum accuracy across all groups and is commonly used to evaluate the model’s robustness against spurious correlations. Overall prediction accuracy is also reported to demonstrate minimum impacts of our approach on ID predictions.

\subsection{Synthetic Data}
\paragraph{Data generating process.} We design $T=3$ domains for training and one unseen domain for testing. We follow a similar data-generating procedure demonstrated in \citep{lu2021nonlinear}: we have a common label-related parameter $C$ to generate invariant features for data in all four domains. For each domain, there is a domain-specific environment $E_i$, $i=1,2,3,4$. For each data point $\xv$, we assume there are three non-trainable functions extracting three different types of features, respectively. The first type is invariant feature $\zv_1(\xv)\in \mathbb{R}^d$, which only depends on $C$. The second $\zv_2(\xv)\in\mathbb{R}^d$ is named nuanced features generated by both $C$ and $E_i$ so it has a weak correlation to the label. The third feature $\zv_3(\xv)\in \mathbb{R}^{k\times d}$ is spurious and generated by $E_i$ only. Here, $k$ is a hyper-parameter that controls the dimension of spurious feature and we choose $k=3$ in the experiment. 
Consequently, the representation has dimension $(k+2)\times d$.
 
\paragraph{Model and objectives.} For simplicity, we set $W=[1,1,\dots,1]$ that is not trainable and a linear representation $\Phi$. Specifically, we define $$\Phi(\xv)=[\av_1^\top\odot\zv_1(\xv)^\top,\av_2^\top\odot\zv_2(\xv)^\top,\av_3^\top\odot\zv_3(\xv)^\top],$$
where $\odot$ is the element-wise product.
Denote $\av=[\av_1^\top,\av_2^\top,\av_3^\top]^\top.$ The ground true label is generated by a representation $\Phi^*(\xv)=\av^*\odot\zv(\xv)$ plus a random noise, where $\av_3^*=0$. 
We provide the data generating process in the appendix.
The nuclear- and Frobenius-norms are reduced to $\ell_1$- and $\ell_2$-norms of $\av$, respectively.
The objective function for training is 
$$\min_{\av}\frac{1}{2nT}\sum_{t=1}^T\sum_{i=1}^n\left(y_{ti}-f(\xv_{ti})\right)^2+R(\av).$$
Our goal is small mean squared errors (MSEs) in the unseen domain.
In the experiment, we consider three different forms of the regularizer $R(\av)$: $\lambda\|\av\|_1$, $\lambda\|\av\|_2^2$, and $\lambda_1\|\av\|_1+\lambda_2\|\av\|_2^2$.
We expect that a $\Phi$ with more non-zero elements in the representation of invariant features and more zero elements for spurious features leads to a better performance on OOD predictions. 

\paragraph{Results.}
We optimize the loss with the three different forms of $R(\av)$ and without $R(\av)$ (i.e., ERM), respectively. We run the simulation 50 times independently and compare the MSE of the training set, ID testing set, and OOD set.
The result is shown in Table \ref{tbl:synthetic mse}. 
Unsurprisingly, ERM has the lowest training MSE but the test error is significantly larger than using the regularized objectives for both ID and OOD tests. Notebaly, using both $\ell_1$ and $\ell_2$ penalties achieves the smallest ID and OOD test errors, and performance is consistent as reflected by the smallest standard errors. 

\begin{table}[!t]
\centering
\vspace{-2mm}
\caption{The MSE (mean $\pm$ standard error) for different objectives on training data, ID test data, and OOD data. The best in OOD generalization is highlighted in bold. The results are averaged over 50 trials.}
\label{tbl:synthetic mse}
\resizebox{1\linewidth}{!}{
\begin{tabular}{l|ccc}
\toprule
& Training  & ID test  & OOD \\
\midrule
ERM & $0.0009_{\pm 0.0005}$ & $29.30_{\pm 10.56}$ & $63.90_{\pm 23.64}$ \\
$\ell_1$ regularization & $0.22_{\pm 0.03}$ & $3.29_{\pm 0.69}$ & $12.82_{\pm 4.60}$ \\
$\ell_2$ regularization & $0.10_{\pm 0.01}$ & $3.59_{\pm 0.79}$ & $13.62_{\pm 4.29}$ \\
$\ell_1+\ell_2$ & $0.17_{\pm 0.02}$ & $\mathbf{3.16_{\pm 0.67}}$ & $\mathbf{11.77_{\pm 3.83}}$ \\
\bottomrule
\end{tabular}}
\vspace{-2mm}
\end{table}

We also examined  $\av_1,\av_2,\av_3$ learned by different objectives. In particular, we compare the proportion of zero elements for each type of features between using $\ell_1$ regularization alone and using the $\ell_1$ and $\ell_2$. The result is presented in Table~\ref{tbl:synthetic coefs}. The average number of zero elements from $\ell_1$ regularized loss is larger for all the three types of features. Using both $\ell_1$ and $\ell_2$ helps extract more information from invariant and nuanced features but more spurious features are also captured, implying a trade-off between preserving label-related features and eliminating environmental features. One can address this issue by mannually adjust $\lambda$'s, and we will show their impacts, shortly.

\begin{table}[!t]
\centering
\caption{The average proportion of zero elements for different types of features among 50 trials. The optimized features from $\ell_1$ regularization is sparser than $\ell_1+\ell_2$ for all kinds of features.}
\label{tbl:synthetic coefs}
\resizebox{1\linewidth}{!}{
\begin{tabular}{l|ccc}
\toprule
Feature & Invariant  & Nuanced & Spurious \\
\midrule
$\ell_1$ regularization & $0.493_{\pm 0.044}$ & $0.259_{\pm 0.044}$ & $0.676_{\pm 0.023}$ \\
$\ell_1+\ell_2$ & $0.348_{\pm 0.043}$ & $0.168_{\pm 0.036}$ & $0.560_{\pm 0.023}$ \\
\bottomrule
\end{tabular}}
\end{table}

\begin{table*}[h!]
\caption{The worst-group and average accuracy (\%) of ElRep compared with state-of-the-art methods. The best worst-group accuracy is highlighted in \textbf{bold}. The best average accuracy is also highlighted in bold if the worst-group accuracy is the same for multiple methods. Performance is evaluated on the test set with models early stopped at the highest worst-group accuracy on the validation set. N/A means no result is reported for UW on CivilComments, therefore we do not test our approach for this particular setting.} \label{tbl:main}
\centering
\resizebox{.8\linewidth}{!}{
\begin{tabular}{l|cc|cc|cc}
    \toprule
    \multirow{2}{*}{{Method}} & \multicolumn{2}{c|}{{Waterbirds}} & \multicolumn{2}{c|}{{CelebA}} & \multicolumn{2}{c}{{CivilComments}} \\
    & {Worst} & {Average} & {Worst} & {Average} & {Worst} & {Average} \\
    \midrule
    ERM & 70.0$_{\pm 2.3}$ & 97.1$_{\pm 0.1}$ & 45.0$_{\pm 1.5}$ & 94.8$_{\pm 0.2}$ & 58.2$_{\pm 2.8}$ & 92.2$_{\pm 0.1}$ \\
    UW & 88.0$_{\pm 1.3}$ & 95.1$_{\pm 0.3}$ & 83.3$_{\pm 2.8}$ & 92.9$_{\pm 0.2}$ & N/A & N/A\\
    Subsample & 86.9$_{\pm 2.3}$ & 89.2$_{\pm 1.2}$ & 86.1$_{\pm 1.9}$ & 91.3$_{\pm 0.2}$ & 64.7$_{\pm 7.8}$ & 83.7$_{\pm 3.4}$ \\
    GroupDRO & 86.7$_{\pm 0.6}$ & 93.2$_{\pm 0.5}$ & 86.3$_{\pm 1.1}$ & 90.5$_{\pm 0.3}$ & 69.4$_{\pm 0.9}$ & 89.6$_{\pm 0.5}$ \\
    PDE & 90.3$_{\pm 0.3}$ & 92.4$_{\pm 0.8}$ & 91.0$_{\pm 0.4}$ & 92.0$_{\pm 0.5}$ & 71.5$_{\pm 0.5}$ & 86.3$_{\pm 1.7}$ \\
    \midrule
    ERM+ElRep   & 79.8$_{\pm 0.7}$ & 89.5$_{\pm 0.7}$ & 52.6$_{\pm 1.4}$ & 95.5$_{\pm 0.2}$ & 60.5$_{\pm 1.6}$ & 91.5$_{\pm 0.2}$ \\
    UW+ElRep   & 89.1$_{\pm 0.5}$ & 92.5$_{\pm 0.3}$ & 90.2$_{\pm 0.7}$ & 92.4$_{\pm 0.3}$ & N/A & N/A \\
    Subsample+ElRep   & 88.7$_{\pm 0.3}$ & 90.8$_{\pm 0.7}$ & 89.6$_{\pm 0.3}$ & 91.1$_{\pm 0.5}$ & 70.8$_{\pm 0.5}$ & 82.1$_{\pm 0.5}$ \\
    GroupDRO+ElRep   & 88.8$_{\pm 0.7}$ & 92.9$_{\pm 0.7}$ & \textbf{91.4}$_{\pm 1.0}$ & \textbf{92.8}$_{\pm 0.2}$ & 70.5$_{\pm 0.5}$ & 79.0$_{\pm 0.7}$\\
    PDE+ElRep  & \textbf{90.4}$_{\pm 0.2}$ & 91.6$_{\pm 0.7}$ & \textbf{91.4}$_{\pm 0.5}$ & 92.4$_{\pm 0.3}$ & \textbf{71.7}$_{\pm 0.2}$ & 80.7$_{\pm 0.9}$ \\
    \bottomrule
\end{tabular}}
\end{table*}
\subsection{Real Data}
\paragraph{Datasets.}(1) \textbf{CelebA} is  comprised of 202,599 face images. We use it for hair-color classification with gender as the spurious feature. The smallest group is blond-hair men, which make up only 1\% of total data, and over 93\% of blond-hair examples are women. (2) \textbf{Waterbirds} is crafted by placing birds \citep{wah_branson_welinder_perona_belongie_2011} on land or water backgrounds \citep{7968387}. The goal is to classify birds as landbirds or waterbirds, and the spurious feature is the background. The smallest group is waterbirds on land. (3) \textbf{CivilComments-WILDS} is used to classify whether an online
comment is toxic or not, and the label is spuriously correlated with mentions of eight demographic identities (DI), i.e. (male, female, White, Black, LGBTQ, Muslim, Christian, other religions). There are 16 group combinations, i.e., (DI, toxic) and (DI, non-toxic).

\paragraph{Baseline Models.} Extant group robustness methods can be categorized into train-once and train-twice, as discussed in Section~\ref{sec:literature}. The former often relies on the results from a single run.  
The latter, such as \citep{liu2021just}, requires running the training procedure in two stages to achieve ideal performance. In this paper, we compare the proposed ElRep against several state-of-the-art train-once methods, but ours is also readily combined with the train-twice approaches.  
Apart from standard ERM, we integrate the ElRep into several state-of-the-art methods, including Upweighting (UW) that inversely reweights group losses by group sizes, GroupDRO \citep{sagawa2019distributionally} that directly optimizes the worst group loss, the more recent PDE \citep{deng2024robust} that trains on a balanced subset of data then progressively expands the training set, and Subsample \citep{deng2024robust}, a simplified version of PDE without the expansion stage. We compare the performance of these methods with and without ElRep.

\paragraph{Experiment Setup.}\label{exp:setup}
We strictly follow the training and evaluation protocols used for the three datasets in previous studies \citep{piratla2022focuscommongoodgroup, deng2024robust}. The experiments are implemented based on the WILDS package \citep{koh2021wildsbenchmarkinthewilddistribution} which uses pretrained ResNet-50 model \citep{he2015deepresiduallearningimage} from Torchvision for CelebA and Waterbirds, and pretrained Bert model \citep{devlin2019bertpretrainingdeepbidirectional} from HuggingFace for CivilComments-WILDS. All experiments were conducted on a single NVIDIA RTX 8000 GPU with 48GB memory. Our code is available at \url{https://github.com/TaoWen0309/ElRep}.

We follow previous work and run all experiments with three different random seeds and report the mean and standard deviation of worst-group and average accuracy. For a fair comparison, the baseline performance is the original results from recent studies \citep{10.5555/3618408.3619982, deng2024robust, phan2024controllableprompttuningbalancing}. We have not modified their published code or hyper-parameters except for adding the regularization. Also, we do not report the performance of UW on the CivilComments dataset since it has not been benchmarked by extant work.
We select the hyper-parameters for the nuclear and Frobenius norms by cross-validation with candidate $\lambda_1$ between $10^{-4}$ and $10^{-3}$ and candidate $\lambda_2$ between $10^{-5}$ and $10^{-4}$.

\subsubsection{Average and Worst-Group Accuracy}
We compare the performance of ERM, UW, Subsample, GroupDRO, and PDE with and without ElRep and report in Table~\ref{tbl:main} their average and worst-group prediction accuracy. 
As a result, the proposed ElRep improves all the state-of-the-art methods compared in worst-group accuracy (the top half versus the bottom half of the table), demonstrating its effectiveness in group robustness.
The best worst-group accuracy is achieved by  GroupDRO or PDE together with ElRep.  
The improvement is more pronounced if ElRep is combined with a more naive model. For example, ERM has been improved by $6.6$ percentage on average. We show how much these extant methods are improved by ElRep in the left panel of Figure~\ref{fig:bar_accu}. 

Furthermore, ElRep helps reduce performance fluctuation. Specifically, the standard deviation of the worst-group accuracy is typically smaller when a method is combined with ElRep, suggesting its consistently effective learning of invariant features, which may be indispensable for domain generalization. 
Although enhanced group robustness is often achieved at the cost of reduced overall accuracy, we observe that ElRep simultaneously improves both average and worst-group accuracy for several baselines on the waterbirds and CelebA datasets, which is shown in the right panel of Figure~\ref{fig:bar_accu}. This is attributed to the theoretical underpinning that ElRep does not undermine ID prediction, as shown in Section~\ref{sec:thy}, shortly.

\begin{figure}[!t]
\includegraphics[width=\linewidth]{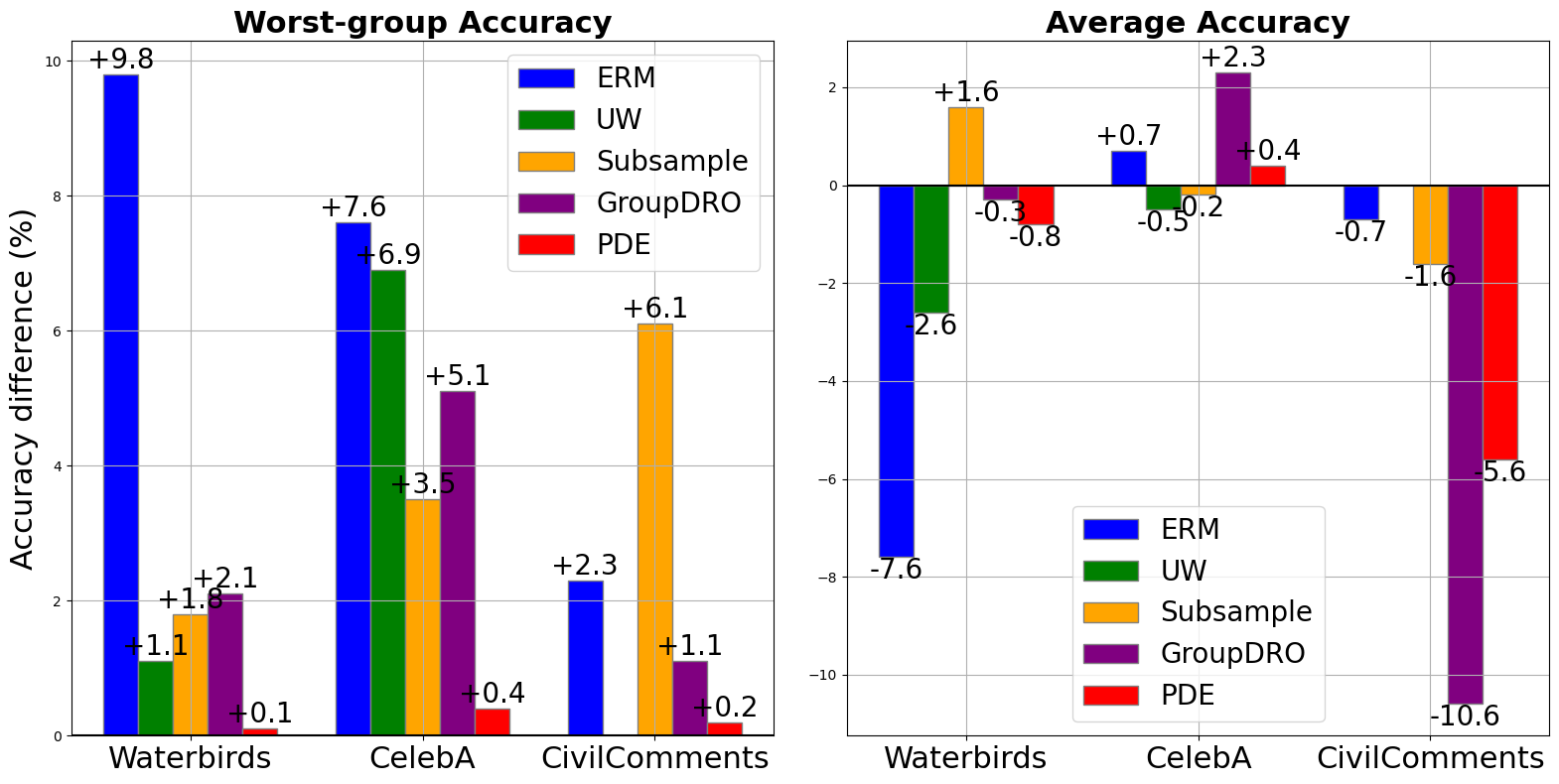}
\caption{Left: The difference in the worst-group accuracy between the baseline methods with and without ElRep. The improvement is ubiquitous among all the methods compared on all the three datasets.  
Right: The difference in the average accuracy between the baseline methods with and without ElRep. Usually, an increase in worst-group accuracy comes with a decrease in average accuracy. Our approach can also improve the average accuracy for some baselines on the image datasets.
}\label{fig:bar_accu}
\end{figure}

\begin{table*}[h!]
\caption{The worst-group and average accuracy (\%) of our approach compared with nuclear Norm (NN) or Frobenius Norm (FN) alone. The experiment settings are the same as in Table \ref{tbl:main}. ElRep achieves the best worst-group performance in almost all settings. 
}\label{tbl:ablation}
\centering
\resizebox{.75\linewidth}{!}{
\begin{tabular}{l|cc|cc|cc}
    \toprule
    \multirow{2}{*}{{Method}} & \multicolumn{2}{c|}{{Waterbirds}} & \multicolumn{2}{c|}{{CelebA}} & \multicolumn{2}{c}{{CivilComments}} \\
    & {Worst} & {Average} & {Worst} & {Average} & {Worst} & {Average} \\
    \midrule
    ERM (FN) & 78.0$_{\pm 0.3}$ & 89.0$_{\pm 0.2}$ & 43.9$_{\pm 4.0}$ & 95.5$_{\pm 0.1}$ & 58.9$_{\pm 1.0}$ & 91.6$_{\pm 0.1}$\\
    ERM (NN) & 78.8$_{\pm 0.3}$ & 89.6$_{\pm 0.5}$ & 44.1$_{\pm 4.7}$ & 95.5$_{\pm 0.1}$ & 59.3$_{\pm 0.2}$ & 91.9$_{\pm 0.2}$\\
    ERM (Ours) & \textbf{79.8}$_{\pm 0.4}$ & 89.5$_{\pm 0.4}$ & \textbf{52.6}$_{\pm 0.8}$ & 95.5$_{\pm 0.1}$ & \textbf{60.5}$_{\pm 0.9}$ & 91.5$_{\pm 0.1}$ \\
    \midrule
    UW (FN) & 88.2$_{\pm 0.6}$ & 92.1$_{\pm 0.1}$ & 89.4$_{\pm 0.5}$ & 92.5$_{\pm 0.2}$ & \multicolumn{2}{c}{\multirow{3}{*}{N/A}} \\
    UW (NN) & 88.4$_{\pm 0.6}$ & 92.0$_{\pm 0.1}$ & 89.7$_{\pm 0.3}$ & 92.2$_{\pm 0.3}$ & & \\
    UW (Ours) & \textbf{89.1}$_{\pm 0.3}$ & 92.5$_{\pm 0.2}$ & \textbf{90.2}$_{\pm 0.4}$ & 92.4$_{\pm 0.2}$ & & \\
    \midrule
    Subsample (FN) & \textbf{89.1}$_{\pm 0.3}$ & 90.9$_{\pm 0.4}$ & 87.8$_{\pm 0.5}$ & 91.9$_{\pm 0.2}$ & 70.3$_{\pm 0.4}$ & 81.2$_{\pm 0.4}$\\
    Subsample (NN) & 88.7$_{\pm 0.1}$ & 91.0$_{\pm 0.3}$ & 88.9$_{\pm 0.5}$ & 91.3$_{\pm 0.1}$ & 70.5$_{\pm 0.3}$ & 80.5$_{\pm 0.3}$\\
    Subsample (Ours) & 88.7$_{\pm 0.2}$ & 90.8$_{\pm 0.4}$ & \textbf{89.6}$_{\pm 0.2}$ & 91.1$_{\pm 0.3}$ & \textbf{70.8}$_{\pm 0.3}$ & 82.1$_{\pm 0.3}$ \\
    \midrule
    GroupDRO (FN) & 88.7$_{\pm 0.5}$ & 92.5$_{\pm 0.3}$ & 90.8$_{\pm 0.2}$ & 93.1$_{\pm 0.1}$ & 69.9$_{\pm 0.5}$ & 78.2$_{\pm 0.5}$\\
    GroupDRO (NN) & 86.8$_{\pm 0.9}$ & 92.4$_{\pm 0.4}$ & 90.8$_{\pm 1.0}$ & 92.8$_{\pm 0.3}$ & \textbf{70.5}$_{\pm 0.5}$ & \textbf{79.2}$_{\pm 0.7}$\\
    GroupDRO (Ours) & \textbf{88.8}$_{\pm 0.4}$ & 92.9$_{\pm 0.4}$ & \textbf{91.4}$_{\pm 0.6}$ & 92.8$_{\pm 0.1}$ & \textbf{70.5}$_{\pm 0.3}$ & 79.0$_{\pm 0.4}$\\
    \midrule
    PDE (FN) & 89.8$_{\pm 0.1}$ & 91.4$_{\pm 0.1}$ & 90.2$_{\pm 0.4}$ & 91.7$_{\pm 0.2}$ & 70.2$_{\pm 0.1}$ & 80.8$_{\pm 0.7}$\\
    PDE (NN) & 89.8$_{\pm 0.2}$ & 91.2$_{\pm 0.3}$ & \textbf{91.4}$_{\pm 0.3}$ & 91.9$_{\pm 0.3}$ & 71.0$_{\pm 0.3}$ & 82.2$_{\pm 0.5}$\\
    PDE (Ours) & \textbf{90.4}$_{\pm 0.1}$ & 91.6$_{\pm 0.4}$ & \textbf{91.4}$_{\pm 0.3}$ & \textbf{92.4}$_{\pm 0.2}$ & \textbf{71.7}$_{\pm 0.1}$& 80.7$_{\pm 0.5}$ \\
    \bottomrule
\end{tabular}}
\end{table*}

\subsubsection{Ablation Study of the Regularization}
\paragraph{Regularization by either nuclear- or Frobenius-norm. }
The advantage of ElRep comes from the combination of a nuclear norm and a Frobenius norm. We consider only using either of them and compare the performance. As reported in Table~\ref{tbl:ablation}, in most cases, our approach is the best. Removing either norm would lead to a degradation of worst-group accuracy, and sometimes, it even underperforms the method without regularization, like ERM on CelebA. 
In addition, our results show that using one norm does not consistently outperform using the other.

\paragraph{Regularization via Weight Decay.}
Though intuitively similar to the elastic net, we regularize the representation rather than the weights. We compare the proposed method with weight decay (WD), which imposes an $\ell_2$ penalty on the weights of the linear classification layer of a neural network. 

We leave CivilComments out for a fair comparison because the Bert model uses its own learning schedule, and magnified weight decay can undermine its performance.
The results in Table \ref{tbl:wd} indicate that ours is better than regularization on the weights in group robustness at a minimum cost of average accuracy. 

\begin{table}[!t]
\caption{The accuracy (\%) of ERM with weight decay (WD) and ElRep. ElRep significantly outperforms WD in worst-group performance with minimal sacrifice of average accuracy.}\label{tbl:wd}
\centering
\resizebox{\linewidth}{!}{ 
\begin{tabular}{l|cc|cc}
    \toprule
    \multirow{2}{*}{{Method}} & \multicolumn{2}{c|}{{Waterbirds}} & \multicolumn{2}{c}{CelebA}\\
    & {Worst} &  {Average} &  {Worst} &  {Average}\\
    \midrule
    ERM+WD & 78.9$_{\pm 0.6}$ & 89.7$_{\pm 0.6}$ & 44.8$_{\pm 3.4}$ & 95.8$_{\pm 0.1}$ \\
    \midrule
    ERM+ElRep & 79.8$_{\pm 0.4}$ & 89.5$_{\pm 0.4}$ & 52.6$_{\pm 0.8}$ & 95.5$_{\pm 0.1}$ \\
    \bottomrule
\end{tabular}}
\end{table}

\subsubsection{Regularization Intensity 
}
\label{exp:sen}
We study the influence of the regularization intensities. Specifically,  $\lambda_1$ and $\lambda_2$ control the nuclear-norm and Frobenius-norm penalties, respectively, and their values affect the model performance. Too small values cannot effectively regularize spurious correlations, while too large values make the penalties overwhelm the classification loss. In Figure~\ref{fig:lambda}, we try various values of $\lambda$ within a reasonable range on CelebA, and show the accuracy on each group and the average accuracy. An obvious trend can be observed that the minority-group (blonde hair) accuracy gradually increases with the value of $\lambda_1$ or $\lambda_2$. If $\lambda$ is sufficiently large the minority group accuracy would eventually surpass the average accuracy. The opposite trend can be observed for the majority groups (non-blonde females and males).

\begin{figure}[!t]
\includegraphics[width=\linewidth]{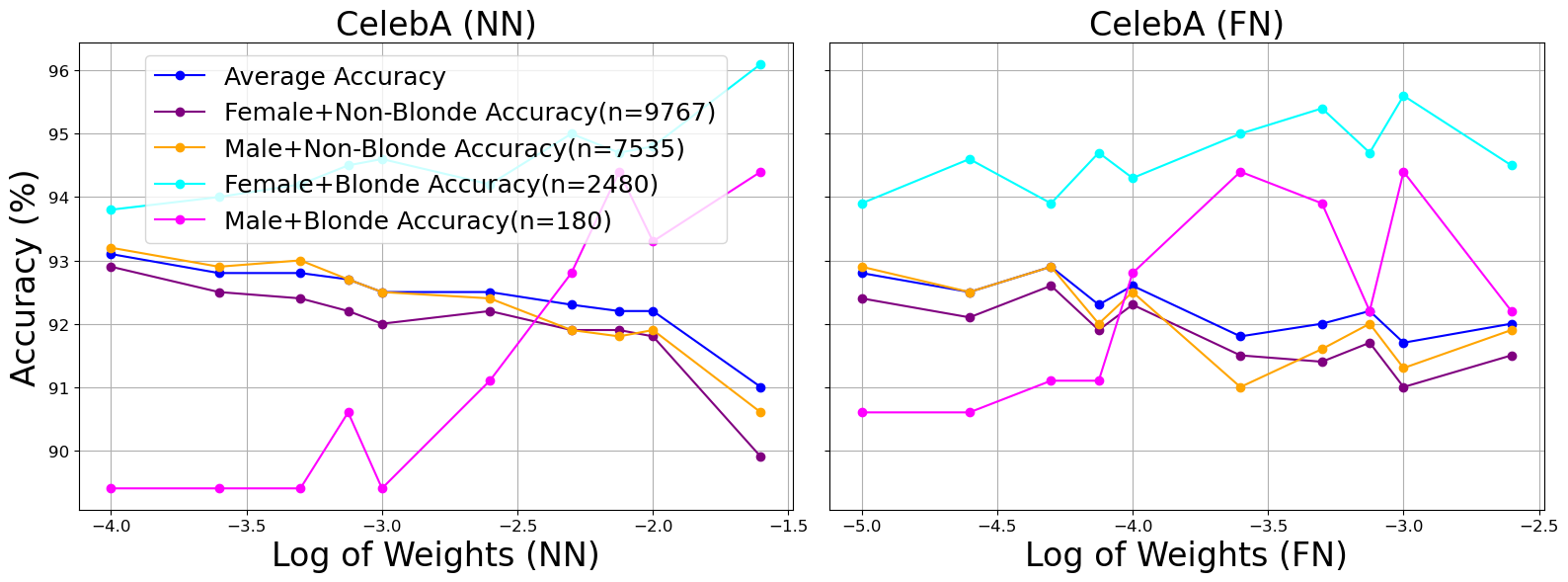}
\caption{Accuracy per group and average accuracy against the log of $\lambda_1$ (left) and $\lambda_2$ (right). As their value increases, the accuracy of the two minority groups will gradually increase and eventually surpass the average accuracy. The trend is reversed for the two majority groups.}\label{fig:lambda}
\end{figure}

To further validate this observation, we randomly downsample the original majority groups, i.e. non-blonde-hair females and males to approximately 1\%. By Figure~\ref{fig:lambda_sub}, we can observe the same trend although the roles of majority and minority groups are now switched compared to  Figure \ref{fig:lambda}. This observation is useful in cases where we only care about small group accuracy since we can set arbitrarily large values for $\lambda_1$ and $\lambda_2$, as long as the regularization term does not overwhelm the classification loss.

\begin{figure}[!t]
\includegraphics[width=\linewidth]{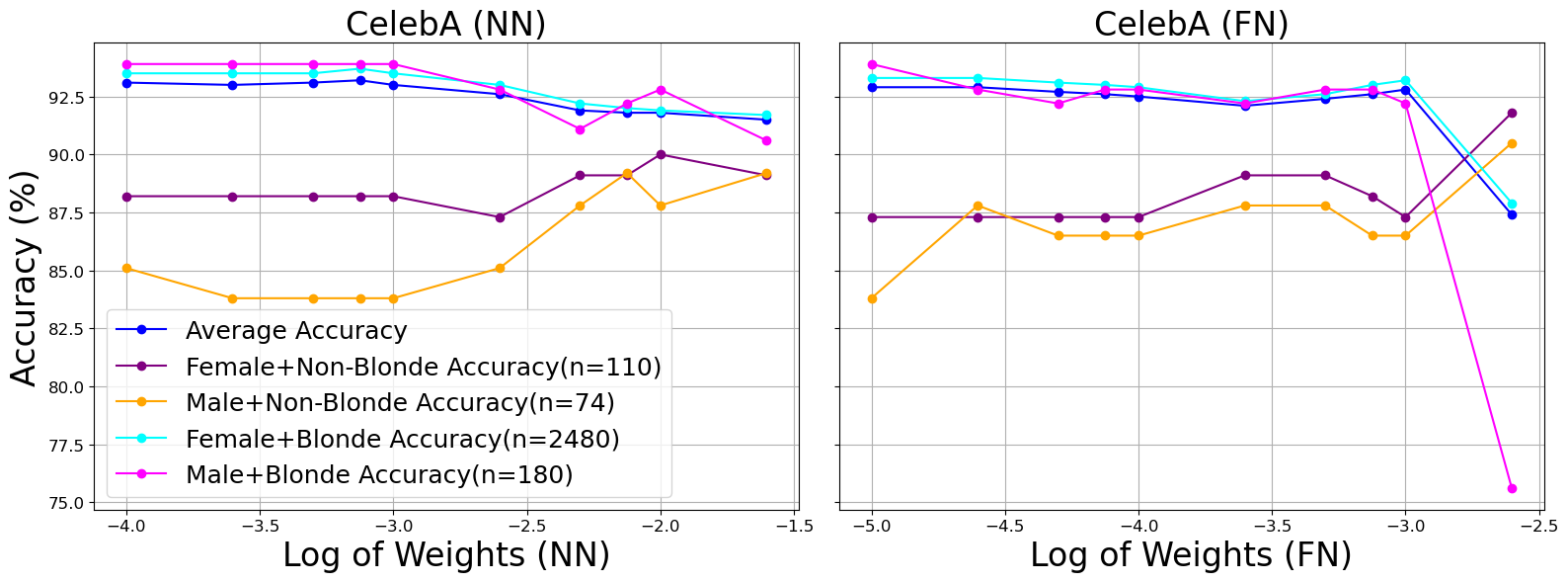}
\caption{The two majority groups downsampled to about 1\%. Reversed trends are observed. 
}\label{fig:lambda_sub}
\end{figure}

\section{THEORETICAL ANALYSIS}\label{sec:thy}
In this section, we provide some theoretical analysis to ElRep, showing that 1) the regularization term will not hurt ID prediction and 2) adding a Frobenius norm term towards nuclear norm penalty can effectively capture more invariant features.

\subsection{ID Prediction}
When training deep learning models, regularization is used to prevent overfits. Previous sections illustrated that ElRep makes OOD prediction more accurate by introducing nuclear- and Frobenius-norm penalties, mitigating an over-regularization of invariant features. However, regularization may hurt ID performance. In this section, we show that the regularization of ElRep does not hurt ID prediction. 

\paragraph{Settings.} 
In our analysis, we consider a regression problem on space $\mathcal{X}\times\mathcal{Y}$, where $\mathcal{X}\subseteq \mathbb{R}^d$ and $\mathcal{Y}\subseteq\mathbb{R}$. We set the model be a linear regression problem $f(\xv)=\theta^\top \xv$ for simplicity. In multi domain learning, there are $T$ different training domains. For each domain, every sample in $X_t\in\mathbb{R}^{n\times d}$ is generated from a distribution $p_t$ supported on $\mathcal{X}$. We assume that $\mathbb{E}_{\xv\sim p_t}\xv=0$ and $\mathbb{E}_{\xv\sim p_t}\xv\xv^\top=\Sigma_t$. Then for $\bar{\xv}=\Sigma^{-1/2}\xv$ generated from $\bar{p}_t$, $\mathbb{E}_{\bar{\xv}\sim\bar{p}_t}\bar{\xv}\bar{\xv}^\top=I.$ The labels $Y_t\in\mathbb{R}^n$ is generated by $Y_t=X_t\theta^*+\epsilon$, where $\Theta^*$ is the ground truth parameter and $\epsilon\sim\mathcal{N}(0,\sigma I_n)$.
\begin{assum}
\label{assum:data}
    There exists a positive semi-definite matrix $\Sigma$ such that $\Sigma_t\preceq \Sigma$ for any $t$.
\end{assum}

\begin{assum}
    \label{assum:subgaussian}
    There exists $\rho>0$ such that the random vector $\bar{\xv}\sim \bar{p}_t$ is $\rho^2$-subgaussian for any $t$.
\end{assum}

\paragraph{Objective.} In the multi-task regression with ElRep, we minimize the following objective
\begin{equation}
\label{eq:objective}
\begin{split}
       \min_{\theta\in \mathbb{R}^d}&\frac{1}{2nT}\|\mathcal{X}(\theta)-Y\|_F^2+\lambda_1\|\theta\|_1+\lambda_2\|\theta\|_2,
\end{split}
\end{equation}
where $\mathcal{X}(\theta)=[X_1\theta,\cdots, X_T\theta]\in\mathbb{R}^{n\times T}$. Note that we penalize both $l_1$ and $l_2$ norm of the regression weight $\theta$, which has a similar effect of penalizing the representation in representation learning setting. 

\paragraph{Theoretical results.} If we denote the solution of \eqref{eq:objective} by $\hat{\theta}$, we are interested in the population excess risk, i.e. $\frac{1}{2T}\sum_{t=1}^T\mathbb{E}_{ p_t}\|X\Delta\|_\mathrm{F}^2$, where $\Delta=\hat{\theta}-\theta^*.$ The following theorem gives an upper bound.
\begin{thm}
\label{thm:ID bound}
    Under Assumption \ref{assum:data} and \ref{assum:subgaussian}, we fix a failure probability $\delta$ and choose proper $\lambda_1,\lambda_2,\lambda_3$. Then with probability at least $1-\delta$ over training samples, the prediction difference between our approach and the ground truth satisfies:
    \begin{equation}
        \frac{1}{2T}\sum_{t=1}^T\mathbb{E}_{ p_t}\|X\Delta\|_\mathrm{F}^2\le \Tilde{O}\left(\frac{\sigma R\sqrt{\mathrm{Tr}(\Sigma)}}{\sqrt{nT}}\right)+\Tilde{O}\left(\frac{\rho^4R^2\mathrm{Tr}(\Sigma)}{nT}\right),
    \end{equation}
    where $R=\|\theta^*\|_1$ and we omit logarithmic factors.
\end{thm}
\looseness-1 The proof of Theorem \ref{thm:ID bound} is deferred to the appendix. This upper bound shows that prediction using ElRep is close to the ground truth if the number of samples $n$ is large, implying ElRep does not hurt ID performance. Note that for nuclear norm regularization, the bound only differs in constant coefficients according to \citet{du2020few}. The analysis of OOD performance is not included because more assumptions of the testing domain are needed, and we defer it to future work.

\subsection{Feature Selection}
\label{subsec:toy_theory}
Nuclear norm regularization improves the OOD prediction by eliminating spurious features. However, the experiments in Section \ref{sec:exp} show that ElRep performs better than the nuclear norm penalty in worst group prediction. One reason is that nuclear norm regularization rules out some invariant features highly correlated with others. In OOD tasks, the correlation may be changed and the eliminated features can be vital in prediction. In this section, we show that ElRep is more likely to keep correlated features than the nuclear norm penalty.

\paragraph{Settings.} For simplicity, we consider a linear regression problem $f(\xv)=\theta^\top\xv$. The training sample $X\in\mathbb{R}^{n\times d}$ has zero mean and satisfies that empirical variance $\frac{1}{n}X^\top X=I_d+\rho(\ev_i\ev_j^\top+\ev_j\ev_i^\top)$, where $\ev_i$ is the $i$-th standard basis vector and $0<\rho<1$. Note that there is a positive correlation $\rho$ between the $i$-th and the $j$-th entry of the data, which is a simplified setting of correlated features. With the ground truth parameter $\theta^*$ and noise $\epsilon\sim\mathcal{N}(0,\sigma I_n)$, the label is generated by $Y=X^\top \theta^*+\epsilon$. We introduce the unregularized least square solution $\hat{\theta}:=\arg\min\|X\theta-Y\|^2$ satisfying $X^\top X\hat{\theta}=X^\top Y$ for the ease of presentation and assume $0<\hat{\theta}_i<\hat{\theta}_j$ without loss of generality.

\paragraph{Theoretical results.} If we denote the least square solution with $\ell_1$ norm regularization by 
$$\theta^1=\underset{\theta\in\mathbb{R}^d}{\arg\min}\frac{1}{2n}\|X\theta-Y\|_2^2+\lambda_1\|\theta\|_1$$
and the least square solution with $\ell_1+\ell_2$ regularizers by
$$\theta^{\mathrm{El}}=\underset{\theta\in\mathbb{R}^d}{\arg\min}\frac{1}{2n}\|X\theta-Y\|_2^2+\lambda_1\|\theta\|_1+\frac{\lambda_2}{2}\|\theta\|_2^2,$$
we have the following result.
\begin{prop}
\label{thm:feature capture}
    Under the following conditions on regularizers $\lambda_1$, $\lambda_2$ and unregularized least square solution $\hat{\theta}$, the regularized least square solutions $\theta^1$ and $\theta^{\mathrm{El}}$ satisfy:
\begin{table}[h]
    \centering
    \renewcommand{\arraystretch}{1.5}
    \resizebox{\columnwidth}{!}{
    \begin{tabular}{|c|c|c|}
        \hline
        $\theta$ stands for: & $\ell_1$ regularization ($\theta^1$) & ElRep ($\theta^{\mathrm{El}}$) \\
        \hline
        $\theta_i,\theta_j>0$ & $\lambda_1<(1+\rho)\hat{\theta}_i$ & $\lambda_1<c$ \\
        \hline
        $\theta_i=0, \theta_j>0$ & $(1+\rho)\hat{\theta}_i\le\lambda_1<\hat{\theta}_j+\rho\hat{\theta}_i$ & $c\le\lambda_1<\hat{\theta}_j+\rho\hat{\theta}_i$ \\
        \hline
        $\theta_i=\theta_j=0$ & $\lambda_1\ge\hat{\theta}_j+\rho\hat{\theta}_i$ & $\lambda_1\ge\hat{\theta}_j+\rho\hat{\theta}_i$\\
        \hline
    \end{tabular}}
\end{table}
    
    where $c=\frac{(1+\lambda_2-\rho^2)\hat{\theta}_i+\lambda_2\rho\hat{\theta}_j}{1+\lambda_2-\rho}.$
\end{prop}
See the appendix for the proof of Proposition \ref{thm:feature capture}. We note that $c>(1+\rho)\hat{\theta}_i$ always holds as we assumed $\hat{\theta}_i<\hat{\theta}_j$ WLOG. Therefore the proposition indicates that ElRep always keeps the features when they are both selected by Lasso: as long as $\theta^1_i,\theta_j^1>0$, we always have $\theta^{\mathrm{El}}_i,\theta^{\mathrm{El}}_j>0$. In contrast, there exists cases when $\theta^{\mathrm{El}}_i,\theta^{\mathrm{El}}_j>0$ while $\theta^1_i=0$. This result indicates that ElRep is more likely to capture correlated features simultaneously. Moreover, since $c-(1+\rho)\hat{\theta}_i$ is larger when $\rho$ and $\lambda_2$ are larger, this contrast of feature selection is more significant with highly correlated features and intense Frobenius norm regularization.

\section{CONCLUSION}\label{sec:conclusion}
In conclusion, we propose to address spurious correlations by Elastic Representation. It enables neural networks to learn more invariant features by imposing the nuclear norm and Frobenius norm of the feature representations and can be readily integrated into a wide range of extant approaches. Theoretically, we show that adding the regularization will not hurt the in-distribution performance. Empirically, extensive experiments validate the proposed method.

\section*{Acknowledgments}
This material is based upon work supported by the U.S. Department of Energy,
Office of Science Energy Earthshot Initiative as part of the project ``Learning reduced models under extreme data conditions for design and rapid decision-making in complex systems" under Award
\#DE-SC0024721.

\bibliographystyle{apalike}
\bibliography{frobnorm}

\clearpage
\onecolumn
\appendix
\section{Details of synthetic data experiment}
In the synthetic data experiment, we generated 3 training domains and 1 testing domain. In the data generating process, we consider a label-related parameter $C\in \mathbb{R}^d$, and for each domain, there is an environmental parameter $E_i\in\mathbb{R}^d$. The features $\zv$ are generated from those parameters. Specifically, the invariant feature $\zv_1=c_1C+\epsilon_1\in \mathbb{R}^d$. The nuanced feature $\zv_2=c_2(\rho C+\sqrt{1-\rho^2}E_i)+\epsilon_2\in \mathbb{R}^d$, where $\rho$ is a hyperparameter controlling the ratio of two types of parameters. The spurious feature $\zv_3=E\cv_3+\epsilon_3$. Here $c_1,c_2\in\mathbb{R}$, $\cv_3\in \mathbb{R}^{1\times k}$ are random coefficients and $\epsilon_1,\epsilon_2\in \mathbb{R}^d$, $\epsilon_3\in \mathbb{R}^{d\times k}$ are random noise.
As we mentioned in Section 4.1, we choose the dimension of spurious feature $k=3$. Moreover, we set $\rho=0.5$, $d=100$ and $n=120$.
\section{Proof of Theorem 5.1}
In order to prove the in-distribution generalization result in Theorem 5.1, we first give some lemmas showing the bound for training error $\mathcal{X}(\Delta)$, where $\mathcal{X}(\theta)=[X_1\theta,\cdots,X_T\theta]\in \mathbb{R}^{n\times T}$. We denote the total noise by $Z:=\mathcal{X}(\theta^*)-Y$, where each column $Z_t\sim\mathcal{N}(0,\sigma I_n)$. 
\begin{lem}
\label{lem:bernstein}
    If Assumption 5.1 holds, then with probability at least $1-\delta$
    $$\frac{1}{nT}\left\|\mathcal{X}^*(Z)\right\|_2\le \Tilde{O}\left(\frac{\sigma\sqrt{\mathrm{Tr}(\Sigma)}}{\sqrt{nT}}\right),$$
    where $\mathcal{X}^*(Z)=\sum_{t=1}^T X_t^\top Z_t$ and the log terms are omitted.
\end{lem}

\begin{proof}
    Let $$A=\frac{1}{\sqrt{n}}\mathcal{X}^*(Z)=\frac{1}{\sqrt{n}}\sum_{t=1}^T X_T^\top Z_t=:\sum_{t=1}^T S_t.$$
    Then we have 
    \begin{equation*}
        \begin{split}
            \mathbb{E}\left[AA^\top\right]&=\mathbb{E}_X\left[\sum_{t=1}^T\frac{1}{n}X_t^\top\mathbb{E}\left[Z_tZ_t^\top\right]X_t\right]\\
            &=\sigma^2\sum_{t=1}^T\Sigma_t
        \end{split}
    \end{equation*}
    and 
    \begin{equation*}
        \begin{split}
            \mathbb{E}\left[A^\top A\right]&=\frac{1}{n}\sum_{t=1}^T \mathbb{E}_Z\left[Z_t^\top \mathbb{E}_X\left[X_tX_t^\top\right]Z_t\right]\\
            &=\sigma^2\sum_{t=1}^T\mathrm{Tr}(\Sigma_t).
        \end{split}
    \end{equation*}
    Then 
    \begin{equation*}
        \nu(A):=\max\left\{\mathbb{E}\left[AA^\top\right], \mathbb{E}\left[A^\top A\right]\right\}=\sigma^2\sum_{t=1}^T\mathrm{Tr}(\Sigma_t). 
    \end{equation*}
    Let $V(A):=\mathrm{diag}\left(\mathbb{E}\left[AA^\top\right], \mathbb{E}\left[A^\top A\right]\right)$. Then 
    \begin{equation*}
        V(A)=\sigma^2\mathrm{diag}\left(\sum_{t=1}^T \Sigma_t,\sum_{t=1}^T\mathrm{Tr}(\Sigma_t)\right)
    \end{equation*}
    and we define $d(A):=\mathrm{Tr}(V(A))/\|V(A)\|_2=2.$ Besides, by Hanson-Wright inequality, we have
    \begin{equation*}
        \|S_t\|_2^2\le \sigma^2\mathrm{Tr}(\Sigma_t)+\sigma^2\|\Sigma_t\|\log\frac{2}{\delta}+\sigma^2\|\Sigma_t\|_\mathrm{F}\sqrt{\log\frac{2}{\delta}}
    \end{equation*}
    with probability $1-\delta/2$. Since $\|\Sigma_t\|_\mathrm{F}\le \mathrm{Tr}(\Sigma_t)$ and $\Sigma_t\preceq \Sigma$, we have
    \begin{equation*}
        \|S_t\|_2\le\sigma\sqrt{\left(1+\sqrt{\log\frac{2}{\delta}}\right)\mathrm{Tr}(\Sigma)+\|\Sigma\| \log\frac{2}{\delta}}=:L.
    \end{equation*}
    Then by Theorem 7.3.1 in \citep{tropp2015introduction}, with probability $1-\delta/2$,
    \begin{equation*}
        \begin{split}
            \|A\|_2&\lesssim \sigma\sqrt{\log\frac{2}{\delta}\nu(A)\log(d(A))}+\log\frac{2}{\delta}\sigma L \log(d(A))\\
            &\lesssim \sigma \left(\log\frac{2}{\delta}\right)^{3/2}\sqrt{T\mathrm{Tr}(\Sigma)}.
        \end{split}
    \end{equation*}
    Thus, with probability at list $1-\delta$, 
    $$\frac{1}{nT}\|\mathcal{X}^*(Z)\|_2\le\Tilde{O}\left(\frac{\sigma\sqrt{\mathrm{Tr}(\Sigma)}}{\sqrt{nT}}\right).$$
\end{proof}
\begin{lem}
\label{lem:train}
    If Assumption 5.1 holds and choose proper $\lambda_1$, $\lambda_2$ and $\lambda_3$, then with probability at least $1-\delta$,
    $$\frac{1}{2nT}\left\|\mathcal{X}(\Delta)\right\|_\mathrm{F}^2\le\Tilde{O}\left(\frac{\sigma R\sqrt{\mathrm{Tr}(\Sigma)}}{nT}\right)$$
    and the optimal solution $\hat{\theta}$ satisfies
    $$\|\hat{\theta}\|_2\lesssim R,$$
    where $R=\|\Theta^*\|_1$ and the log terms are omitted.
\end{lem}
\begin{proof}
    By the definition of $\hat{\theta}$, we have the following inequality:
    \begin{equation*}
        \frac{1}{2nT}\|\mathcal{X}(\Delta)+Z\|_\mathrm{F}^2+\lambda_1\|\hat{\theta}\|_1+{\lambda_2}\|\hat{\theta}\|_2\le \frac{1}{2nT}\|Z\|_\mathrm{F}^2+\lambda_1\|\theta^*\|_1+\lambda_2\|\theta^*\|_2.
    \end{equation*}
    Then 
    $$\frac{1}{2nT}\|\mathcal{X}(\Delta)\|_\mathrm{F}^2+\frac{1}{nT}\left\langle\mathcal{X}(\Delta),Z\right\rangle+R(\hat{\theta})\le R(\theta^*),$$
    where $R(\theta)=\lambda_1\|\theta\|_1+\lambda_2\|\theta\|_2.$
    By reordering the inequality
    \begin{equation*}
        \begin{split}
            \frac{1}{2nT}\|\mathcal{X}(\Delta)\|_\mathrm{F}^2&\le -\frac{1}{nT}\left\langle\mathcal{X}(\Delta),Z\right\rangle+R(\theta^*)-R(\hat{\theta})\\
            &\le \frac{1}{nT}\left(\|\hat{\theta}\|_2+\|\theta^*\|_2\right)\|\mathcal{X}^*(Z)\|_2+R(\theta^*)-R(\hat{\theta}).\\
        \end{split}
    \end{equation*}
    If we choose $\lambda_1=\frac{\|\mathcal{X}^*(Z)\|_2}{nT}$ and $\lambda_2=\frac{2\|\mathcal{X}^*(Z)\|_2}{nT}$, then 
    $$\frac{1}{2nT}\|\mathcal{X}(\Delta)\|_\mathrm{F}^2+{\lambda_1}\|\hat{\theta}\|_1+\frac{\lambda_2}{2}\|\hat{\theta}\|_2\le \frac{1}{nT}\|\theta^*\|_2\|\mathcal{X}^*(Z)\|_2+R(\theta^*).$$
    The right hand side 
    \begin{equation}
    \begin{split}
         \frac{1}{nT}\|\theta^*\|_2\|\mathcal{X}^*(Z)\|_2+R(\theta^*)&= \frac{1}{nT}\left(\|\theta^*\|_2+\|\theta^*\|_1+2\|\theta^*\|_2\right)\|\mathcal{X}^*(Z)\|_2\\
         &= \frac{1}{nT}\left(4\|\theta^*\|_1\right)\|\mathcal{X}^*(Z)\|_2\\
         &\le \Tilde{O}\left(\frac{\sigma R\sqrt{\mathrm{Tr}(\Sigma)}}{\sqrt{nT}}\right),
    \end{split}
    \end{equation}
    where the last equation applies Lemma \ref{lem:bernstein}. 
    Therefore 
    $$\frac{1}{2nT}\|\mathcal{X}(\Delta)\|_\mathrm{F}^2\le \Tilde{O}\left(\frac{\sigma R\sqrt{\mathrm{Tr}(\Sigma)}}{\sqrt{nT}}\right),$$
    and
    $$\frac{\|\mathcal{X}^*(Z)\|_2}{nT}\|\hat{\theta}\|_2\le\frac{4R}{nT}\|\mathcal{X}^*(Z)\|_2,$$
    implying $\|\hat{\theta}\|_2\lesssim R$.
\end{proof}
With the result of above, we can now proof Theorem 5.1.
\begin{proof}[Proof of Theorem 5.1]
    By Lemma C.10 in \citep{du2020few}, if Assumption 5.2 holds,
    $$\left\|\Sigma_t^{1/2}\Delta\right\|_2\le \frac{1}{\sqrt{n}}\|X_t\Delta\|_2+\frac{C\rho}{\sqrt{n}}\left(\sqrt{\mathrm{Tr}(\Sigma_t)}+\sqrt{\log\frac{2}{\delta}\|\Sigma_t\|}\right)\|\Delta\|_2.$$
    Then 
    \begin{equation*}
    \begin{split}
        \mathbb{E}_{p_t}\|X\Delta\|_2^2=\left\|\Sigma_t^{1/2}\Delta\right\|_2^2&\lesssim \frac{1}{n}\|X_t\Delta\|_2^2+\frac{C\rho^3}{n}\left(\mathrm{Tr}(\Sigma_t)+\log\frac{2}{\delta}\|\Sigma_t\|\right)\|\Delta\|_2^2\\
        &\lesssim \frac{1}{n}\|X_t\Delta\|_2^2+\frac{C\rho^4\log\frac{2}{\delta}\mathrm{Tr}(\Sigma_t)}{n}\left(\|\hat{\theta}\|_2^2+\|\theta^*\|_2^2\right)\\
        &\le \frac{1}{n}\|X_t\Delta\|_2^2+\frac{C\rho^4\log\frac{2}{\delta}\mathrm{Tr}(\Sigma_t)}{n}R^2,
    \end{split}
    \end{equation*}
    where the last inequality is from the second part of Lemma \ref{lem:train}.
    We sum the above inequality up for all $t=1,\dots,T$,
    \begin{align*}
        \frac{1}{2T}\sum_{t=1}^T\mathbb{E}_{p_t}\|X\Delta\|_2^2&\lesssim \frac{1}{2T}\sum_{t=1}^T\left(\frac{1}{n}\|X_t\Delta\|_2^2+\frac{C\rho^4\log\frac{2}{\delta}\mathrm{Tr}(\Sigma_t)}{n}R^2\right)\\
        &=\frac{1}{2nT}\|\mathcal{X}\Delta\|_\mathrm{F}^2+\frac{1}{2T}\sum_{t=1}^T\frac{C\rho^4\log\frac{2}{\delta}\mathrm{Tr}(\Sigma_t)}{n}R^2\\
        &\le \Tilde{O}\left(\frac{\sigma R\sqrt{\mathrm{Tr}(\Sigma)}}{\sqrt{nT}}\right)+\Tilde{O}\left(\frac{\rho^4 R^2\mathrm{Tr}(\Sigma)}{nT}\right),
    \end{align*}
    where the last inequality is given by the first part of Lemma \ref{lem:train}.
\end{proof}

\end{document}